\documentclass[11pt]{article}
\usepackage{fullpage,graphicx,psfrag,amsmath,amsfonts,verbatim}
\usepackage{xcolor}
\usepackage{amsthm}
\usepackage[small,bf]{caption}
\usepackage{authblk}

\usepackage{amsmath}
\usepackage{amssymb}
\usepackage{mathtools}
\usepackage{amsfonts}
\usepackage{algorithm}
\usepackage{algorithmic}

\usepackage{colortbl}
\usepackage{siunitx}
\usepackage{booktabs,makecell,xcolor}
\usepackage{dcolumn}
\usepackage{tcolorbox}
\usepackage{subcaption}

\usepackage{microtype}
\usepackage{graphicx}
\usepackage{booktabs} 
\usepackage{hyperref}

\usepackage{enumitem}

\newtheorem{lemma}{Lemma}
\newtheorem{proposition}{Proposition}

\hypersetup{
    colorlinks = true,
    allcolors = {purple},
    linkbordercolor = {white},
}

\def\approxcorrect{\checkmark\kern-1.1ex\raisebox{.89ex}{$\times$}}

\usepackage{amsmath,amsfonts,bm}

\def\eqref#1{equation~\ref{#1}}

\def\1{\bm{1}}

\DeclareMathAlphabet{\mathsfit}{\encodingdefault}{\sfdefault}{m}{sl}
\SetMathAlphabet{\mathsfit}{bold}{\encodingdefault}{\sfdefault}{bx}{n}

\allowdisplaybreaks

\title{Differentiable Adaptive Kalman Filtering via Optimal Transport}
\author[1]{Yangguang He}
\author[2]{Wenhao Li\thanks{\texttt{whli@tongji.edu.cn}}}
\author[3]{Minzhe Li}
\author[4]{Juan Zhang}
\author[5]{Xiangfeng Wang}
\author[2,1]{Bo Jin\thanks{\texttt{bjin@tongji.edu.cn}}}

\affil[1]{Shanghai Research Institute for Intelligent Autonomous Systems, Tongji University}
\affil[2]{School of Computer Science and Technology, Tongji University}
\affil[3]{School of Aeronautics and Astronautics, Shanghai Jiaotong University}
\affil[4]{Shanghai Zhangjiang Institute of Mathematics}
\affil[5]{School of Computer Science and Technology, East China Normal University}

\date{}
\begin{document}
\maketitle

\begin{abstract}
Learning‐based filtering has demonstrated strong performance in non-linear dynamical systems, particularly when the statistics of noise are unknown. 
However, in real-world deployments, environmental factors, such as changing wind conditions or electromagnetic interference, can induce unobserved noise-statistics drift, leading to substantial degradation of learning-based methods. 
To address this challenge, we propose OTAKNet, the first online solution to noise-statistics drift within learning-based adaptive Kalman filtering. 
Unlike existing learning‐based methods that perform offline fine-tuning using batch pointwise matching over entire trajectories, OTAKNet establishes a connection between the state estimate and the drift via one-step predictive measurement likelihood, and addresses it using optimal transport.
This leverages OT’s geometry‑aware cost and stable gradients to enable fully online adaptation without ground‑truth labels or retraining.
We compare OTAKNet against classical model-based adaptive Kalman filtering and offline learning-based filtering. The performance is demonstrated on both synthetic and real-world NCLT datasets, particularly under limited training data. 
\end{abstract}

\section{Introduction}

State estimation lies at the heart of real‐time control and navigation applications.
Learning-based filters combine data-driven learning with model-based inference, representing a promising direction for state estimation.
Most notably, KalmanNet \cite{2022TSP_KalmanNet} and its extensions in reinforcement learning \cite{AAAI2025_Kalman_fusion_in_RL}, brain–machine interfaces \cite{NEURIPS2024_f1cf02ce}, multimodal data assimilation \cite{2024CVPR_deepgeneartivemutimodal}, human‐motion capture \cite{NEURIPS2024_KNet_human_motion}, temporal recommendation \cite{24WSDM_NeuFilter} and video super‐resolution \cite{2024ECCV_Kalman_videoface}. 

In practical deployments, environmental and operational factors, such as sensor temperature variations, mechanical wear, wind gusts, or electromagnetic interference, can alter the underlying Gaussian noise-statistics. 
This leads to a mismatch between the true noise covariances (NC) and those assumed during filtering, a phenomenon known as noise uncertainty, referred to in this paper as noise-statistics drift, which has been extensively studied in classical model-based adaptive Kalman filtering (AKF) \cite{huang2017novel, huang2020slide}, encompassing correlation methods, covariance-matching techniques, maximum-likelihood approaches, and Bayesian schemes \cite{mehra2003approaches}.

Unfortunately, learning-based approaches have rarely explored the issue of noise-statistics drift, typically addressing it by enlarging training datasets \cite{ni2024adaptive} and performing offline fine-tuning \cite{chen2025maml} via batch pointwise loss, i.e., pointwise matching.
This raises the question of how to extend current offline pointwise learning-based methods to enable online adaptation?

Unfortunately. First, using it only once offers no memory of noise drift. 
Second, single-sample gradients exhibit high variance and lack convergence guarantees \cite{2023TAC_Logarithmic_Regret}. Third, pointwise matching entangles model bias with noise, making it difficult to separately tune predictor and noise drift.
Finally, treating each time step independently breaks the recursive, causal structure central to Kalman filtering \cite{2017adaptiveKF_elec_power_system}.
Consequently, these issues render pointwise loss inadequate for online adaptation under noise-statistics drift.
This gives rise to two fundamental challenges: (1) how to formally characterize the impact of noise-statistics drift on state estimation, and (2) how to quantify the degree of noise-statistics drift.

To solve noise-statistic drift, we propose OTAKNet, a differentiable online adaptive Kalman filtering via optimal transport (OT). 
To address challenge (1), from a distributional perspective, the one-step predictive measurement likelihood is used to formally characterize the impact of noise-statistics drift on state estimation. Specifically, at each time step, a source distribution is constructed from the filter’s predictive estimates to approximate the one-step measurement likelihood under noise-statistics drift, while a target distribution is built from current observation and recent innovations to retain temporal information about the drift.
To address challenge (2), we quantify the degree of noise-statistics drift via the differentiable OT distance \cite{cuturi2013sinkhorn}, which captures the geometric discrepancy between the source and target distributions, shaped by temporal information about the drift. In contrast, alternatives such as KL divergence neglect the geometric structure of the sample space and often suffer from vanishing gradients under large noise excursions.

Overall, our contributions include: 
(1) We present the first online solution to noise-statistics drift within learning-based adaptive Kalman filtering, by proposing OTAKNet.
(2) To address noise-statistics drift, a connection between the state estimate and the drift via one-step predictive measurement likelihood is established via optimal transport.
(3) Empirically, our OTAKNet consistently outperforms both model-based and learning-based baselines on synthetic and real-world datasets, particularly under limited training data or during highly maneuvering scenarios.

\section{Related Work}

\paragraph{Model-based Adaptive Kalman Filters}

To address noise‐statistics drift, classical adaptive kalman filtering (AKF) methods perform online estimation of filter noise parameters based on various adaptation rules. In particular, by imposing inverse Wishart priors within a variational Bayesian (VB) framework, the state estimate, prediction‐error covariance, and measurement‐noise covariance matrices are inferred jointly \cite{huang2017novel}. 
However, this VBAKF approach strongly depends on a preselected nominal process noise covariance matrix (SNCM).
To enhance robustness, a slide window variational adaptive Kalman filter (SWVAKF) was proposed \cite{huang2020slide}, which incorporates a backward Kalman smoother (KS) over a finite window to iteratively refine the noise covariance estimates.

\paragraph{Learning–Based Adaptive Kalman Filters}

Several recent works have explored learning noise characteristics for improved filtering performance. EKFNet \cite{2024TSP_EKFNet} and \cite{NIPS2023_offline_SGD_DD_filter} both learn the unknown NC matrices from offline data, but are not specifically designed to mitigate noise‐statistics drift.

Building on the KalmanNet architecture, AKNet \cite{ni2024adaptive} introduces a compact hypernetwork to generate context-dependent modulation weights for noise adaptation. However, it relies on abundant supervised trajectory data and requires prior knowledge of noise-parameter ratios at deployment.
Meta‐learning aided methods such as MAML‐KalmanNet \cite{chen2025maml} employ a tailored pre-training scheme to enable rapid fine‐tuning with only a few target‐domain samples under noise drift.
However, these methods rely on fully labeled offline data to address noise‐statistics drift and thus cannot support real-time noise adaptation during filtering.

\paragraph{Optimal Transport for Filtering}

Optimal transport provides a principled and geometrically grounded framework for distribution alignment, and is a powerful tool for quantifying distributional differences. 
It has demonstrated remarkable success in unsupervised domain adaptation~\cite{2017_PAMI_UDA_OT, 2023_CVPR_UDA_framwork, 2024_AAAI_UDA_ImageRetrieval}.

Beyond these computer vision adaptation tasks, OT has been extended to dynamic state‑space estimation. For example, \cite{ICML2021_OT_RSPF} reformulates the resampling step in particle filtering as an entropy-regularized OT problem, yielding a differentiable filter with low‑variance gradient. 
\cite{ICML2024_NL_OTPF} models the Kalman predict‑update cycle itself as an OT map and designs an OT-based particle filter transporting prior particles to the posterior.  

Despite these advances, most learning‑based filtering methods are trained offline with fixed noise covariance assumptions. When noise statistics drift during deployment, this challenge can be formulated as a domain adaptation problem induced by distributional shift. However, OT‑based solutions for online adaptation in learning‑based Kalman filters remain to be explored.  

\section{Preliminaries}

\subsection{Adaptive Kalman Filter}

Adaptive Kalman Filter is a widely used technique for state estimation under uncertainty in the noise covariances. In this work, we focus on uncertainty arising from \textbf{noise‑statistics drift}. Consider the state‑space model (SSM):
\begin{subequations}
\setlength{\abovedisplayskip}{3pt}
\setlength{\belowdisplayskip}{3pt}
\begin{align}
  x_t &= \mathbf{f}(x_{t-1}) + w_t,\quad &w_t &\sim \mathcal{N}(0,\,\mathbf{Q}),\\
  y_t &= \mathbf{h}(x_t) + v_t,\quad &v_t &\sim \mathcal{N}(0,\,\mathbf{R}),
\end{align}
\end{subequations}
where \(x_t\in\mathbb{R}^n\) is the latent state at time \(t\), \(y_t\in\mathbb{R}^m\) the corresponding measurement, \(\mathbf{f}\) and \(\mathbf{h}\) are state‑evolution and measurement functions, and \(w_t\), \(v_t\) are Gaussian noises with unknown covariances \(\mathbf{Q}\) and \(\mathbf{R}\), respectively.

In classical AKF, one alternately estimates the noise covariances \(\hat{\mathbf{Q}}_t\) and \(\hat{\mathbf{R}}_t\), and then computes the posterior using the update equations with these estimated covariances. We now introduce a key building block of OTAKNet, the one‑step predictive likelihood of the measurement:
\begin{equation}
  p(y_t\mid y_{1:t-1})
  = \int p(y_t\mid x_t)\,p(x_t\mid y_{1:t-1})\,\mathrm{d}x_t.
\end{equation}
Under the Gaussian assumptions, the time‑update (prior) and likelihood take the form
\begin{subequations}
\begin{align}
  p(x_t\mid y_{1:t-1})
  &= \mathcal{N}\bigl(x_t;\,\hat x_{t\mid t-1},\,\Sigma_{t\mid t-1}\bigr), \label{equa_kf_predict_a}\\
  p(y_t\mid x_t)
  &= \mathcal{N}\bigl(y_t;\,\mathbf{h}(x_t),\,\hat{\mathbf{R}}\bigr), \label{equa_kf_update_b}\
\end{align}
\end{subequations}
where $\hat x_{t\mid t-1} = \mathbf{f}\bigl(\hat x_{t-1}\bigr),\,\Sigma_{t\mid t-1} = \mathbf{F}_t\,\Sigma_{t-1}\,\mathbf{F}_t^\top + \hat{\mathbf{Q}},\,\mathbf{F}_t = \nabla_x\mathbf{f}\bigl(\hat x_{t-1}\bigr)$. Substituting into the integral yields the closed‑form predictive measurement distribution:
\begin{equation}\label{equa_predictive_meas}
  p(y_t\mid y_{1:t-1})
  = \mathcal{N}\bigl(y_t;\,\mathbf{h}(\hat x_{t\mid t-1}),\,\mathbf{S}_{t\mid t-1}\bigr), 
\end{equation}
where \( \mathbf{S}_{t\mid t-1} = \mathbf{H}_t\,\Sigma_{t\mid t-1}\,\mathbf{H}_t^\top + \hat{\mathbf{R}},\,\mathbf{H}_t = \nabla_x\mathbf{h}\bigl(\hat x_{t\mid t-1}\bigr)\).    
\subsection{Optimal transport}

Optimal transport provides a geometry-aware approach for measuring distribution discrepancy. Given two probability measures \(\mu\) and \(\nu\) on a metric space \(\mathcal{X}\) and a cost function \(c(x,y)\), the Kantorovich OT problem \cite{kantor_ot} is  
\begin{equation}\label{Kantorovich OT}
    W_c(\mu,\nu)
  = \min_{\pi\in\Pi(\mu,\nu)}
    \int_{\mathcal{X}\times\mathcal{X}} c(x,y)\,\mathrm{d}\pi(x,y),
\end{equation}
where \(\Pi(\mu,\nu)\) denotes the set of all couplings with marginals \(\mu\) and \(\nu\). In the special case \(c(x,y)=\|x-y\|^2\), the resulting distance is the squared Wasserstein‑2 distance that respects the geometry of \(\mathcal{X}\) and captures both support and mass differences, making it well-suited for measuring distributional shifts.

Directly solving the OT problem (\ref{Kantorovich OT}) is computationally expensive. To enable efficient approximation, one can add an entropic regularization term \(\varepsilon>0\), yielding the so-called Sinkhorn distance \cite{cuturi2013sinkhorn}:
\begin{equation}\label{sinkhorn dis}
  W_{\varepsilon}(\mu,\nu)
  = \min_{\pi\in\Pi(\mu,\nu)}
    c(x,y)\,\mathrm{d}\pi(x,y)
    - \varepsilon\,h(\pi).
\end{equation}

The entropic regularizer $h(\pi) = \sum_{i,j} \pi_{ij}\,\ln \pi_{ij}$ yields an
an optimization problem (\ref{sinkhorn dis}) solvable via iterative Bregman projections \cite{2015iterative_entropic_ot}. 
In practice, the Sinkhorn algorithm \cite{cuturi2013sinkhorn} alternates updates to compute the regularized coupling \(\pi\) in \(O(n^2)\) time per iteration, and is fully differentiable with respect to distribution parameters. An Inexact Proximal point method for exact Optimal Transport problem (IPOT) algorithm \cite{2020UAI_IPOT}, refines this approach and converges to the exact Wasserstein distance. These properties make them practical for integration into deep learning pipelines and online adaptation schemes.  

\section{Method}

To address noise‐statistics drift, we propose OTAKNet, a differentiable, label-free, OT-based online adaptive Kalman filter. In essence, we characterize the impact of noise-statistics drift on state estimation at each time step via the one-step predictive measurement likelihood. 
We then mitigate the degree of noise-statistics drift by using the optimal transport plan to 'move' predictive likelihood. Over time, this adaptation converges towards the true measurement likelihood.
Specifically, we use two empirical distributions: a source distribution, derived from prior neural filter estimates, and a target distribution constructed from current observation and recent innovations to explicitly retain temporal information about the drift. 
The Wasserstein distance quantifies the degree of noise-statistics drift on the predictive measurement distribution.
To ensure real-time filtering and differentiability, we employ the IPOT algorithm to efficiently compute the OT loss and update parameters in each predict–update cycle.
The framework is shown in Fig.~\ref{fig1_method}.

\begin{figure*}[ht]
  \centering
  \includegraphics[width=0.9\linewidth]{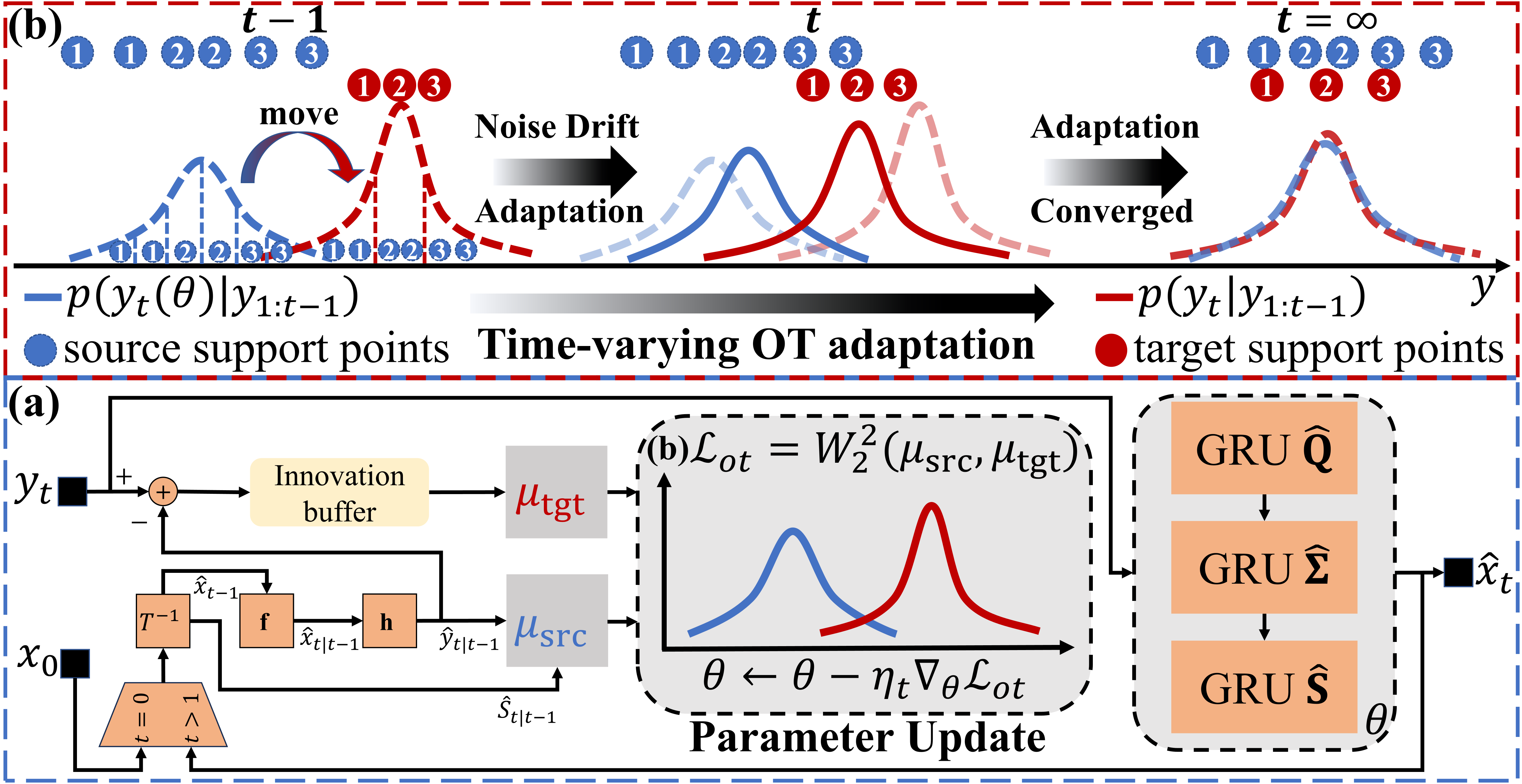}
  \caption{(a) OTAKNet filtering framework. At each time step \(t\), the prior defines a source distribution to approximate the one-step measurement likelihood under noise-statistics drift, while the current measurement and a window of past innovations define the target. The differentiable OT-based adaptation (illustrated in panel (b)) is used to update parameters \(\theta\), yielding the corrected estimate \(\hat x_t\).  
  (b) The source iteratively moves toward the target distribution. Over successive time steps, these two measures converge to the true predictive likelihood, enabling fully online label-free adaptation to noise-statistics drift.}
  \label{fig1_method}
\end{figure*}

\subsection{Problem Formulation for Noise-Statistics Drift}

In many real-world scenarios, true noise covariances evolve dynamically due to external factors such as wind gusts or electromagnetic interference, violating the fixed-noise assumption and degrading the performance of offline-trained learning-based kalman filtering.

During offline training, label-free (i.e., not using ground-truth states) learning-based filters (e.g., \cite{TSP_2024danse, 2024TSP_EKFNet}) aim to maximize the one-step predictive likelihood defined in Eq.\ref{equa_predictive_meas}. In practice, this is typically achieved by minimizing the mean square error between the network’s predicted measurement \(\hat y_t(\theta)\) and the true measurement \(y_t\):
\begin{equation}\label{equ_offline_loss}
  \min_{\theta} \sum_{t=1}^T \mathbb{E} \bigl\| \hat{y}_t(\theta) - y_t \bigr\|^2.
\end{equation}

This objective aligns estimated pointwise on complete offline trajectories. However, applying such pointwise matching loss in an online setting to handle noise-statistics drift introduces several issues.
First, single-sample gradients exhibit high variance and lack convergence guarantees \cite{2023TAC_Logarithmic_Regret}.
Second, treating each time step independently disrupts the recursive and causal structure inherent to Kalman filtering \cite{2017adaptiveKF_elec_power_system}.
Consequently, pointwise matching loss is inadequate for online adaptation under noise-statistics drift.

To overcome these limitations, unlike prior approaches that rely on pointwise estimates, our formulation operates in the distributional space. We establish a principled connection between noise-statistics drift and state estimation using the predictive measurement distribution in Eq.~\ref{equa_predictive_meas}. 
Optimal‐transport (OT) distances quantify the minimal cost of transforming one distribution into another while preserving the geometry of the measurement space, and they yield stable, non‐degenerate gradients. 
Accordingly, we adopt the Wasserstein-2 distance as our alignment criterion to detect and correct noise-statistics drift.
Our adaptive objective is then formulated as
\begin{equation}\label{equ_otloss}
  \theta_t^* = \arg\min_{\theta} \;
W_2^2\bigl( p(y_t(\theta) \mid y_{1:t-1}),\; p(y_t \mid y_{1:t-1}) \bigr)
\end{equation}
where \(W_{2}\) is the Wasserstein‑2 distance in the measurement space. 
This formulation enables label-free adaptation by aligning the predictive distribution, affected by noise-statistics drift, with the geometry-informed observation.

\subsection{Source Distribution Construction}

The source distribution characterizes how noise-statistics drift affects state estimation through the one-step predictive measurement likelihood distribution. At time \(t\), the offline trained KalmanNet provides the prior state mean \(\hat x_{t\mid t-1}\) and innovation covariance \(S_{t\mid t-1}\).  We approximate the predictive measurement distribution by Monte Carlo sampling. Concretely, we proceed as follows:

\begin{enumerate}
  \item Let \(\hat x_{t|t-1}(\theta)\) and \(S_{t|t-1}(\theta)\) be prior mean and error covariance of the learning-based filter under parameters $\theta$.
  \item Sample \(N\) independent particles using Eq.~\ref{equa_predictive_meas}
  \[
    z_{\mathrm{src}}^{(i)} \;\sim\;\mathcal{N}\bigl(\mathbf{h}(\hat x_{t|t-1}(\theta)),\,S_{t|t-1}(\theta)\bigr),
    \quad i=1,\dots,N.
  \]
  \item Define the \textbf{source} empirical measure
  \begin{equation}\label{equa_src dis}
      \mu_{\mathrm{src}} \;=\; \frac{1}{N}\sum_{i=1}^N \delta_{z_{\mathrm{src}}^{(i)}}.
  \end{equation}
\end{enumerate}

In this way, \(\mu_{\mathrm{src}}\) provides an empirical, parameterized approximation of the predictive measurement distribution based on the offline trained filters parameters $\theta$. 

\subsection{Target Distribution Construction}\label{subsec: target dis}

The target distribution is designed to encode temporal information about the drift. At time step \(t\), we observe only the current \(y_t\). Since a single pointwise cannot retain temporal information about the drift, we utilize the innovations
\begin{equation}\label{innovation}
  e_k = y_k - \mathbf{h}(\hat x_{k\mid k-1}), 
  \quad k = t-W+1,\dots,t,
\end{equation}
and maintain the sliding window \(\{e_{t-W+1},\dots,e_t\}\).  For each \(j=1,\dots,W\), we generate a pseudo‐measurement $ \tilde z_{\rm tgt}^{(j)} = y_t + e_{\,t-W+j}$,
effectively shifting the current observation by each past innovation to capture measurement noise variability. The resulting target empirical measure is
\begin{equation}\label{equa_target dis}
  \mu_{\rm tgt}
  = \frac{1}{W}\sum_{j=1}^W \delta_{\tilde z_{\rm tgt}^{(j)}}.
\end{equation}
This construction enriches the current observation with geometric context to retain temporal information about the drift in Eq. \ref{innovation} and simultaneously defines the target distribution for OT-based adaptation.


\subsection{Differentiable OT Loss}

Let $z_{\mathrm{src}}^{(i)}$ and $\tilde z_{\mathrm{tgt}}^{(j)}$ denote source and target support points, respectively. The cost matrix is defined as $C_{ij} = \tfrac12 \bigl\lVert\,z_{\mathrm{src}}^{(i)}- \tilde z_{\mathrm{tgt}}^{(j)}\bigr\rVert^2$, and the OT problem in Eq.~\ref{sinkhorn dis} becomes
\begin{equation}
  \pi^* = \arg\min_{\pi\in\Pi(a,b)} \sum_{i=1}^N\sum_{j=1}^W \pi_{ij}\,C_{ij} 
    \;-\; \varepsilon\,h(\pi),
\end{equation}
with uniform marginals $a_i = 1/N$ and $b_j = 1/W$. To compute optimal transport plan $\pi^*$ efficiently, we employ the IPOT algorithm \cite{2020UAI_IPOT}, which converges to the exact Wasserstein distance with theoretical guarantee.
Following the IPOT procedure, define the kernel matrix \(G_{ij} = \exp(-C_{ij}/\epsilon)\). At each outer iteration \(k=1,2,\dots\), set \(Q = G \odot \pi^{(k)}\), where \(\odot\) denotes the Hadamard product. Then, we perform \(L\) inner iterations update (typically \(L=1\)) on the scaling vectors:
\begin{equation}
  a \leftarrow \frac{\mu}{Qb}, \qquad 
  b \leftarrow \frac{\nu}{Q^\top a}.
\end{equation}
The coupling is updated as $\pi^{(k+1)} \leftarrow \operatorname{diag}(a)\, Q\, \operatorname{diag}(b)$. 
All within PyTorch so that gradients \(\nabla_\theta\) propagate through each operation. We define the OT loss
\begin{equation}
  \mathcal{L}_{\mathrm{OT}}
  = \sum_{i=1}^N\sum_{j=1}^W \pi_{ij}\,C_{ij},
\end{equation}
which is fully differentiable with respect to \(\theta\).

\subsection{Online OT‐Driven Parameter Update}
The core of our method integrates a differentiable OT framework for addressing online noise-statistics drift in learning-based filters, resulting in a fully label-free, online noise adaptation mechanism. As the filter parameters are iteratively updated, the source and target distributions progressively align and converge toward the true predictive likelihood. 
The proposition and proof of \textit{online innovation covariance adaptive consistency} are provided in the Appendix \ref{Theoretical Analyze}.

To mitigate estimation variability at initialization when temporal information about the drift is limited, i.e., when the residual buffer is small, we employ a linear warm-up schedule for the learning rate:
\begin{equation}\label{equ_warmup}
  \eta_t = \min\!\Bigl(\eta,\;\frac{t}{W}\,\eta\Bigr),
\end{equation}
where \(W\) denotes the window length, $\eta$ denotes the initial learning rate. This schedule reduces initial estimation fluctuations and rapidly reaches the nominal learning rate. Algorithm \ref{alg:otpf} provides the detailed update procedure.  

\begin{algorithm}[tb]
    \caption{OT-based Adaptive Neural Kalman Filter}
    \label{alg:otpf}
     \textbf{Input:} Initial state $x_0$, observations $\{y_t\}_{t=1}^{T}$, system dynamics $\mathbf{f}(\cdot)$, measurement $\mathbf{h}(\cdot)$, offline trained neural filter with parameter $\theta$ \\
    \textbf{Output:} Estimated states $\{\hat x_t\}_{t=1}^T$ \\
    \textbf{Hyperparameters:} regularization $\epsilon$, inner iterations $K$, queue length $W$, learning rate $\eta_t$ \\
    \textbf{Initialize:} $\mathrm{Res} \leftarrow \mathrm{Queue}(\text{maxlen}=W)$
    load network parameters $\theta$
\begin{algorithmic}[1]
    \STATE $\hat x_{1}(\theta),\,\Sigma_{1\mid1}(\theta),\,\hat{\mathbf{Q}}_1(\theta),\,\hat{\mathbf{R}}_1(\theta) \leftarrow \texttt{OTAKNet}_{\theta}(y_{1})$
    \FOR{$t = 2$ \textbf{to} $T$}
        \FOR{$k = 1$ \textbf{to} $K$}
            \STATE {\scriptsize{\texttt{\# Prior prediction}}}
            \STATE $\hat x_{t \mid t-1} \leftarrow \mathbf{f}(\hat x_{t-1})$,\quad $ \Sigma_{t \mid t-1} \leftarrow  \mathbf{F}_t\,\Sigma_{t-1}\, \mathbf{F}_t^\top + \hat{\mathbf{Q}}_{t-1}$
            \STATE {\scriptsize{\texttt{\# Particle propagation }}}
            \STATE $\mathbf{S}_{t\mid t-1} = \mathbf{H}_t\,\Sigma_{t\mid t-1}\,\mathbf{H}_t^\top + \hat{\mathbf{R}}_{t-1}$
            \FOR{$i = 1$ \textbf{to} $N$}
                \STATE $y_{t}^i \sim \mathcal{N}\bigl(\mathbf{h}(\hat x_{t \mid t-1}),\,\mathbf{S}_{t\mid t-1})$
            \ENDFOR
            \STATE {\scriptsize{\texttt{\# Store residuals}}}
            \STATE $\mathrm{Res}.\text{insert} (y_{t} - \mathbf{h}(\hat x_{t \mid t-1}))$
            \STATE $\{\tilde{y}_t^j\}_{j=1}^W \gets \{ y_t + r \mid r \in \mathrm{Res} \}$
            \STATE {\scriptsize{\texttt{\# Construct source and target distributions}}}
            \STATE $\mu \leftarrow \tfrac{1}{N}\sum_{i=1}^N \delta_{y_{t}^i},\quad
                    \nu \gets \frac{1}{W} \sum_{j=1}^W \delta_{\tilde{y}_t^j}$
            \STATE Build cost matrix $C_{ij} = \tfrac12 \|y_{t}^i - \tilde{y}_t^j\|^2$
            \STATE Compute transport plan $P \leftarrow \mathrm{IPOT}(\mu, \nu, C, \epsilon)$
            \STATE Compute OT loss: 
            $\mathcal{L}_{ot} \leftarrow \sum_{i,j} P_{ij} \cdot C_{ij}$
            \STATE Update parameters: 
            $\theta \leftarrow \theta - \eta_t\,\nabla_{\theta}\,\mathcal{L}_{ot}$
        \ENDFOR
        \STATE {\scriptsize{\texttt{\# Inference}}}
        \STATE $ \hat x_{t},\,\Sigma_{t \mid t},\,\mathbf{Q}_t,\,\mathbf{R}_t \leftarrow \texttt{OTAKNet}_{\theta}(y_t)$
    \ENDFOR
\end{algorithmic}
\end{algorithm}

\section{Experiments}
We evaluate the performance of OTAKNet under noise-statistics drift on both synthetic and real-world datasets, including the University of Michigan North Campus Long-Term Vision and LIDAR (NCLT) dataset \cite{dataset_NCLT}, and compare it with model-based and learning-based filtering methods.
An ablation study isolates the contributions of the warm-up schedule and the optimal transport adaptation module. 

\subsection{Noise-Statistics Drift Setting}

Noise‐statistics drift refers to the mismatch between the true test-time noise covariances and the assumed covariances in the state-space model (SSM). 
In the synthetic experiments, we generate multiple test trajectories under different noise covariance regimes to emulate varying degrees of drift
while in the real-world NCLT dataset, the true covariances are unknown and subject to environmental variability.

For the synthetic dataset, we follow the setup of KalmanNet~\cite{2022TSP_KalmanNet} with diagonal noise covariance:
\begin{equation}\label{4.0.3}
    \mathbf{Q} = \mathrm{q}  ^ { 2 } \cdot \mathbf{I} , \mathbf{R} = \mathrm{r} ^ { 2 } \cdot \mathbf{I} , \nu \triangleq \frac{ \mathrm{q} ^ { 2 } } { \mathrm{r} ^ { 2 } } 
\end{equation}
where \(\nu\) is the process‐to‐measurement noise‐variance ratio.  Performance is measured by the mean square error (MSE) on a decibel (dB) scale, and results are reported as a function of the inverse measurement noise level \(1/\mathrm{r}^2\) also on [dB].

\subsection{Baselines}

1) Learning‐based methods: KalmanNet \cite{2022TSP_KalmanNet} and OTAKNet are offline trained in a fixed NC to isolate the effect of online adaptation. To ensure a fair comparison, the offline adaptive methods AKNet~\cite{ni2024adaptive} and MAML‑KalmanNet~\cite{chen2025maml} are instead offline trained across a dense set of noise ratios covering the test‐time regimes. 
\\
2) Model‐based methods: EKF \cite{EKF_book} uses the true covariance values and thus represents an oracle benchmark; VBAKF-PR \cite{huang2017novel} and SWVAKF \cite{huang2020slide} are initialized with the same nominal covariances as KalmanNet, relying on their internal adaptation rules to adjust NC online. 
Both in \textbf{Subsection} "Study on Synthetic Data: Lorenz Attractor" and "Study on Real World Data: NCLT", we set hyperparameters as follows.
For the \textbf{VBAKF-PR} baseline, we set the number of iterations to \(N_{\mathrm{iter}} = 100\), with parameters \(\tau_P = 3\) and \(\tau_R = 3\). The adaptation rate is set as \(\rho = 1 - \exp(-4)\).
For the \textbf{SWVAKF} baseline, we use a adaptation rate \(\rho = 1 - \exp(-5)\), and a sliding window length \(L = 20\).

\subsection{Study on Synthetic Data: Lorenz Attractor}\label{subsec: LOR}

\paragraph{State-space model for Lorenz Attractor}
Lorenz attractor is a three-dimensional chaotic solution arised from the Lorenz system of ordinary differential equations in continuous time \cite{Lorenz2004}. It captures chaotic, sensitive dynamics and serves as a benchmark for evaluating filtering algorithms in nonlinear systems \cite{ICML2024_NL_OTPF, 2022TSP_KalmanNet}. Accordingly, the system is discretized into a state-space model (SSM) as follows:
\begin{subequations}\label{equa_LOR_SSM}
\setlength{\abovedisplayskip}{3pt}
\setlength{\belowdisplayskip}{3pt}
\begin{align}
    {x}_t &= \exp
    \left( 
    \begin{bmatrix}
    -10 & 10 & 0 \\
    28 & -1 & -{x}_{t-1, 1} \\
    0 & {x}_{t-1, 1} & -\frac{8}{3}
    \end{bmatrix}
    \cdot
    \Delta
    \right)
    \cdot{x}_{t-1} + \mathbf{w}_t, \label{4.0.1.1} \\
    {y}_t &= \mathbf{h}({x}_t) + \mathbf{v}_t, \label{4.0.1.2}
\end{align}
\end{subequations}
where the system dynamics is a non-linear function of state $x_{t-1}$, $\mathbf{h}({x}_t)$ is identity (linear measurement), the process noise \( \mathbf{w}_t \) and measurement noise \( \mathbf{v}_t \) are Gaussian (\( \mathbf{w}_t \sim \mathcal{N}(0, \mathbf{Q}) \) , \( \mathbf{v}_t \sim \mathcal{N}(0, \mathbf{R}) \) ). Unless stated otherwise, the data was generated $\Delta=0.02$ sampling interval.

\paragraph{Baseline implement details} 
For fair comparison, we standardize the offline training of all methods under comparable noise covariance settings. 
For the learning-based filters, AKNet is offline trained under four discrete noise ratios \(\nu\in\{-10,0,10,20\}\,\mathrm{dB}\), corresponding to \(1/\mathrm{r}^2\) values of \(\{0.1,1,10,100\}\), while MAML-KalmanNet is trained over a dense grid of 28 noise ratios \(\nu\) ranging from -30 to 30 dB.  By contrast, the KalmanNet backbone in OTAKNet is offline trained only once at the nominal ratio \(\nu=0\) dB (\(1/\mathrm{r}^2=0\) dB) using 1000 trajectories, relying on its OT-driven online adaptation to adapt to all other drift scenarios.
For the model-based filters, we initialized the noise covariance with the same settings as OTAKNet. 

The LOR dataset for OTAKNet an KalmanNet consists of 1000 training trajectories, 100 validation trajectories, and 100 test trajectories. 
During offline training, the model is trained for 2,000 iterations with mini-batches of 64 samples.
During online testing deployment, the OTAKNet model performs test-time adaptation via the Adam optimizer with a learning rate of \(1.8 \times 10^{-3}\) and \(L_2\) weight decay of \(1 \times 10^{-3}\), using a sliding adaptation window \(W = 20\).

\begin{table}[ht]
  \setlength{\abovedisplayskip}{3pt}
  \setlength{\belowdisplayskip}{3pt}
  \centering
  \setlength{\tabcolsep}{3.0mm}
  \begin{tabular}{@{}l *{5}{c}@{}}
    \toprule
    $1/\mathrm{r}^2\,[\mathrm{dB}]$ &
    \cellcolor[HTML]{E2E6E1}\textbf{-10}$\downarrow$ &
    \cellcolor[HTML]{E2E6E1}\textbf{0}$\downarrow$ &
    \cellcolor[HTML]{E2E6E1}\textbf{10}$\downarrow$ &
    \cellcolor[HTML]{E2E6E1}\textbf{20}$\downarrow$ &
    \cellcolor[HTML]{E2E6E1}\textbf{30}$\downarrow$ \\
    \midrule
    \makecell[c]{%
      \begin{tabular}{@{}l@{}l@{}}
        \makebox[2.5cm][l]{OKF} & \shortstack{$\hat{\mu}$\\$\hat{\sigma}$}
      \end{tabular}
    } &
      \makecell[c]{9.64\\$\pm0.29$}   & \makecell[c]{-0.43\\$\pm0.34$} &
      \makecell[c]{-10.00\\$\pm0.35$} & \makecell[c]{-20.40\\$\pm0.36$} &
      \makecell[c]{-30.35\\$\pm0.33$} \\
    \midrule
    \makecell[c]{%
      \begin{tabular}{@{}l@{}l@{}}
        \makebox[2.5cm][l]{VBAKF-PR} & \shortstack{$\hat{\mu}$\\$\hat{\sigma}$}
      \end{tabular}
    } &
      \makecell[c]{42.66\\$\pm4.03$}  & \makecell[c]{16.20\\$\pm3.56$} &
      \makecell[c]{-0.90\\$\pm0.91$}  & \makecell[c]{-11.75\\$\pm0.96$} &
      \makecell[c]{-18.19\\$\pm2.73$} \\
    \midrule
    \makecell[c]{%
      \begin{tabular}{@{}l@{}l@{}}
        \makebox[2.5cm][l]{SWVAKF} & \shortstack{$\hat{\mu}$\\$\hat{\sigma}$}
      \end{tabular}
    } &
      \makecell[c]{15.11\\$\pm1.60$}  & \makecell[c]{2.72\\$\pm0.70$}  &
      \makecell[c]{-5.68\\$\pm0.89$}  & \makecell[c]{-13.46\\$\pm1.06$} &
      \makecell[c]{-22.09\\$\pm1.13$} \\
    \midrule
    \makecell[c]{%
      \begin{tabular}{@{}l@{}l@{}}
        \makebox[2.5cm][l]{AKNet} & \shortstack{$\hat{\mu}$\\$\hat{\sigma}$}
      \end{tabular}
    } &
      \makecell[c]{21.14\\$\pm2.11$}  & \makecell[c]{7.15\\$\pm0.53$}  &
      \makecell[c]{-3.04\\$\pm0.53$}  & \makecell[c]{-13.10\\$\pm0.49$} &
      \makecell[c]{-23.08\\$\pm0.49$} \\
    \midrule
    \makecell[c]{%
      \begin{tabular}{@{}l@{}l@{}}
        \makebox[2.5cm][l]{\shortstack[l]{MAML-\\KalmanNet}} & \shortstack{$\hat{\mu}$\\$\hat{\sigma}$}
      \end{tabular}
    } &
      \makecell[c]{22.46\\$\pm1.70$}  & \makecell[c]{9.75\\$\pm0.81$}  &
      \makecell[c]{-0.41\\$\pm0.73$}  & \makecell[c]{-10.45\\$\pm0.74$} &
      \makecell[c]{-19.56\\$\pm1.31$} \\
    \midrule
    \makecell[c]{%
      \begin{tabular}{@{}l@{}l@{}}
        \makebox[2.5cm][l]{KalmanNet} & \shortstack{$\hat{\mu}$\\$\hat{\sigma}$}
      \end{tabular}
    } &
      \makecell[c]{15.26\\$\pm1.31$}  & \makecell[c]{3.46\\$\pm0.41$}  &
      \makecell[c]{-6.54\\$\pm0.42$}  & \makecell[c]{-16.35\\$\pm0.42$} &
      \makecell[c]{-24.63\\$\pm0.94$} \\
    \midrule
    \makecell[c]{%
      \begin{tabular}{@{}l@{}l@{}}
        \makebox[2.5cm][l]{OTAKNet} & \shortstack{$\hat{\mu}$\\$\hat{\sigma}$}
      \end{tabular}
    } &
      \makecell[c]{\textbf{11.83}\\${\pm0.79}$} &
      \makecell[c]{\textbf{2.28}\\${\pm0.38}$}  &
      \makecell[c]{\textbf{-8.77}\\${\pm0.35}$} &
      \makecell[c]{\textbf{-18.78}\\${\pm0.42}$} &
      \makecell[c]{\textbf{-26.22}\\${\pm1.11}$} \\
    \bottomrule
  \end{tabular}
  \caption{MSE [dB] under the noise‐statistics drift scenario}
  \label{table1_LOR_mse}
\end{table}

\paragraph{Numerical results} 
To evaluate robustness under noise‐statistics drift, we introduce five levels of static drift by varying the true measurement noise \(1/\mathrm{r}^2\) over \(\{30, 20, 10, 0, -10\}\) [dB] to simulate increasing mismatch between assumed and true covariances. For each noise setting, 100 independent test trajectories, i.e., 100 Monte Carlo runs, are generated, each with \(T=100\) time steps.
\\
\textbf{Static drift performance:} 
Static drift performance refers to average MSE across all time steps under noise‐statistics drift. Table \ref{table1_LOR_mse} presents the MSE (in [dB]) at the five drift levels. As drift increases, OTAKNet consistently outperforms both model-based and learning‑based approaches, demonstrating its ability to adapt to noise‐statistics drift.
\\
\textbf{Dynamic drift performance:} 
Dynamic drift performance examines the MSE over each time step as the filter adapts online to a noise drift scenario. Fig. \ref{fig2:LOR_time_mse} shows the MSE curves for \(1/\mathrm{r}^2=20\) [dB] over \(T=100\) steps. Within 25 timesteps, OTAKNet self‐adjusts and approaches the performance of an EKF configured with the true covariances, highlighting its rapid adaptation capability. 

\begin{figure}[ht]
  \centering
  \includegraphics[width=0.8\linewidth]{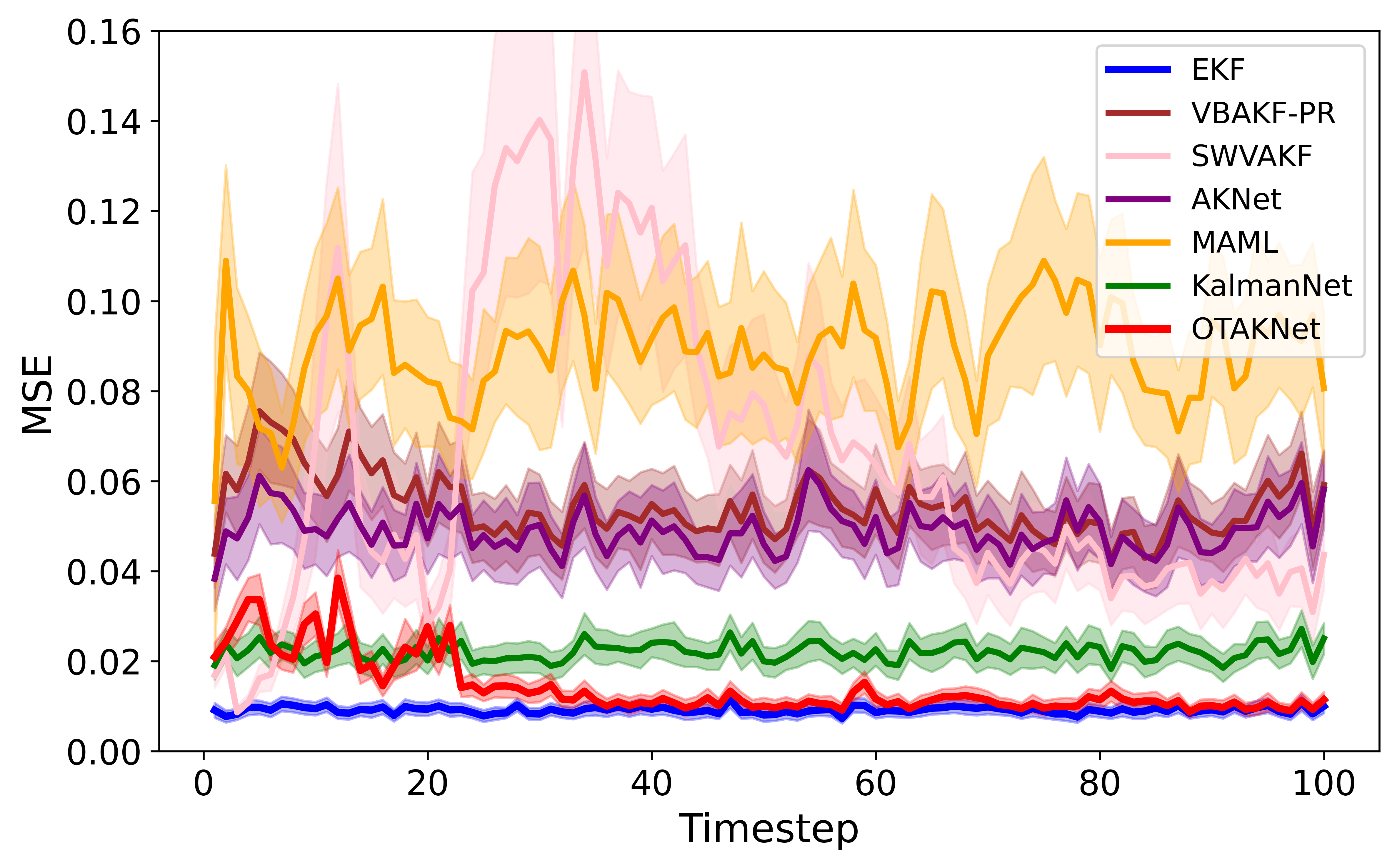}
  \caption{MSE curves over 100 timesteps on synthetic Lorenz Attractor data under the noise shift scenario (\(1/\mathrm{r}^2 = 20\) dB). OTAKNet progressively adapts, converging toward EKF performance with true noise covariances.}
  \label{fig2:LOR_time_mse}
\end{figure}

\begin{figure}[ht]
  \centering
  \includegraphics[width=0.8\linewidth]{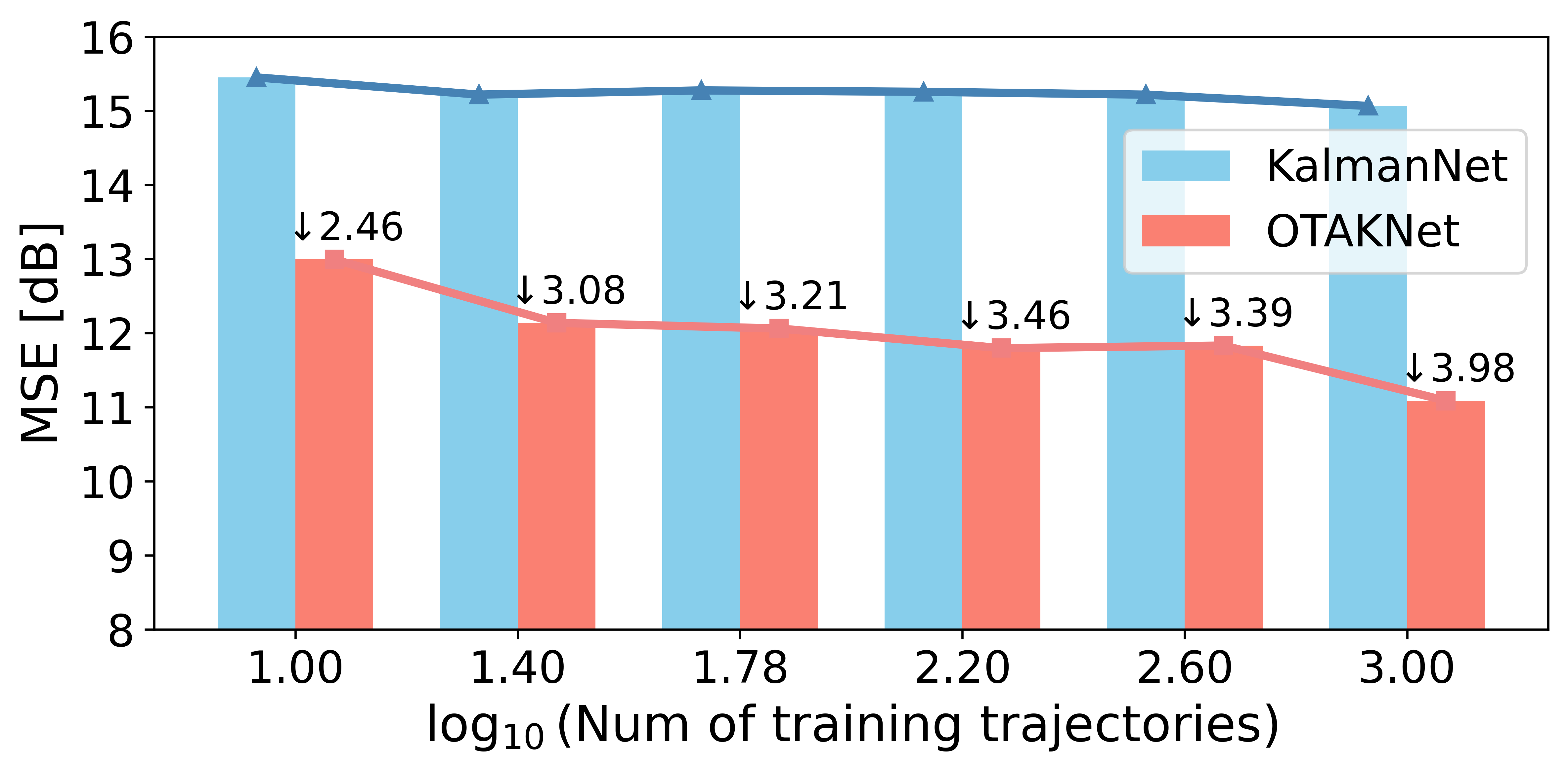}
  \caption{Effect of increasing the number of offline training trajectories in Lorenz Attractor.}
  \label{fig3:LOR_training trajectories}
\end{figure}

\subsection{Study on Real World Data: NCLT}

\paragraph{State-space model for NCLT}

In practical deployments, environmental changes often induce latent drift in noise statistics.
We evaluate OTAKNet on the NCLT dataset, collected using a Segway mobile robot at the University of Michigan North Campus.
This study focuses on localizing the robot by fusing IMU, odometry, and GPS measurements.
We define the five‐dimensional state vector $x_t = [\,x_t,\;y_t,\;v_{x,t},\;v_{y,t},\;\theta_t\,]^\top$ consists of the planar position \((x_t,y_t)\), the Cartesian velocity components \((v_{x,t},v_{y,t})\) computed from the wheel‐odometry speed command \(v_{c,t}\), and the heading angle \(\theta\) measured by the onboard IMU. The discrete‐time system dynamics are modeled as
\begin{subequations}\label{eq:ssm_nclt}
\setlength{\abovedisplayskip}{3pt}
\setlength{\belowdisplayskip}{3pt}
\begin{align}
  x_t &= 
  \begin{bmatrix}
    x_{t-1} + \Delta \,v_{c,t} \cos(\theta) \\[3pt]
    y_{t-1} + \Delta \,v_{c,t} \sin(\theta) \\[3pt]
    v_{c,t} \cos(\theta) \\[3pt]
    v_{c,t} \sin(\theta) \\[3pt]
    \theta_{t}
  \end{bmatrix}
  +
  w_t, \label{eq:ssm_nclt_dyn}\\
  y_t &= 
  \begin{bmatrix}
    1 & 0 & 0 & 0 & 0 \\[3pt]
    0 & 1 & 0 & 0 & 0
  \end{bmatrix}
  x_t
  + v_t, \label{eq:ssm_nclt_meas}
\end{align}
\end{subequations}
where \(w_t\sim\mathcal{N}(0,\mathbf{Q})\) accounts for unmodeled perturbations in the motion model, and \(v_t\sim\mathcal{N}(0,\mathbf{R})\) reflects the uncertainty in GPS observations. The sampling interval $\Delta=1$. 

\paragraph{Baseline implement details}
In real-world deployments, the number of trajectories required by learning-based filters to achieve satisfactory performance is unknown, and collecting sufficient data may be infeasible. To evaluate the robustness of OTAKNet, we consider two offline training regimes, both followed by online fine-tuning:
1) \textbf{Full} training: all 13 labeled trajectories are used for offline training.
2) \textbf{Limited} training: only 3 trajectories are used for offline training.
Notably, AKNet and MAML‑KalmanNet require labeled trajectories for offline fine‑tuning, which may be impractical in real-time deployments. Consequently, our evaluation focuses on online adaptive kalman methods.

The process noise covariance \(\mathbf{Q}\) is defined as a diagonal matrix with stronger uncertainty in the planar position components and minimal perturbation in velocity and heading:
\[
\mathbf{Q} = \mathrm{diag}(1.0,\,1.0,\,0.001,\,0.001,\,0.001),
\]
corresponding to the states \((x,\,y,\,\dot{x},\,\dot{y},\,\theta)\), respectively. The GPS measurement noise is assumed to be isotropic with standard deviation \(10\,\mathrm{m}\), resulting in the covariance:
\[
\mathbf{R}_{\mathrm{GPS}} = 10^2 \cdot \mathbf{I}_2.
\]

\paragraph{Numerical results}
We test on a long trajectory recorded on 2013‑04‑05 by a Segway robot, with ground-truth from high-precision RTK GPS and observations including IMU, wheel odometry, and low-precision GPS.
We extract a 4000-step continuous sequence aligned to RTK-GPS ground truth and divide it into 20 non-overlapping trajectories of 200 steps each: 13 for training, 3 for validation, and 4 for testing. More details please refer to Appendix.
\\
\textbf{Static drift performance:} 
Table \ref{table2_nclt_mse} presents MSE in dB for two different training regimes. 
The results show that VBAKF outperforms the fixed-noise EKF, suggesting that adaptively tuning noise covariance is effective, and environmental variations may lead to latent noise-statistics drift.
OTAKNet demonstrates superior performance of model-based methods on real-world data.
Furthermore, under limited training, the proposed OTA framework mitigates KalmanNet’s degradation under limited training.
\\
\textbf{Dynamic drift performance:} 
Fig. \ref{fig:nclt_two_panel} shows the MSE curves of the filtering process on the NCLT dataset over \(T=200\) steps.
For most of the time, OTAKNet adapts noise statistics in real time and outperforms other online methods, even under limited training. Besides, at the beginning and end of the estimated trajectory in Fig.~\ref{fig:nclt_NCLT_trajectory_comparison_xy}, the two maneuvers performance highlights OTAKNet’s ability to adapt to abrupt dynamic changes and reduce estimation errors.

\begin{table}
  \centering
  \setlength{\tabcolsep}{0.35mm}
  \begin{tabular}{@{}l *{5}{c}@{}}
    \toprule
     & \textbf{EKF} & \textbf{VBAKF} & \textbf{SWVAKF} & \textbf{KalmanNet} & \textbf{OTAKNet} \\
    \midrule
    \cellcolor[HTML]{E2E6E1}Full
      & \makecell[c]{7.51\\\(\pm 2.55\)}
      & \makecell[c]{7.12\\\(\pm 2.42\)}
      & \makecell[c]{8.45\\\(\pm 1.75\)}
      & \makecell[c]{6.35\\\(\pm 2.15\)}
      & \makecell[c]{\textbf{6.10}\\\(\pm 2.84\)} \\
    \midrule
    \cellcolor[HTML]{E2E6E1}Limited
      & – & – & – 
      & \makecell[c]{9.40\\\(\pm 1.88\)}
      & \makecell[c]{\textbf{7.10}\\\(\pm 3.03\)} \\
    \bottomrule
  \end{tabular}
  \caption{MSE [dB] on the NCLT dataset}
  \label{table2_nclt_mse}
\end{table}

\begin{figure}
  \centering
  \begin{subfigure}[b]{0.8\linewidth}
    \centering
    \includegraphics[width=\linewidth]{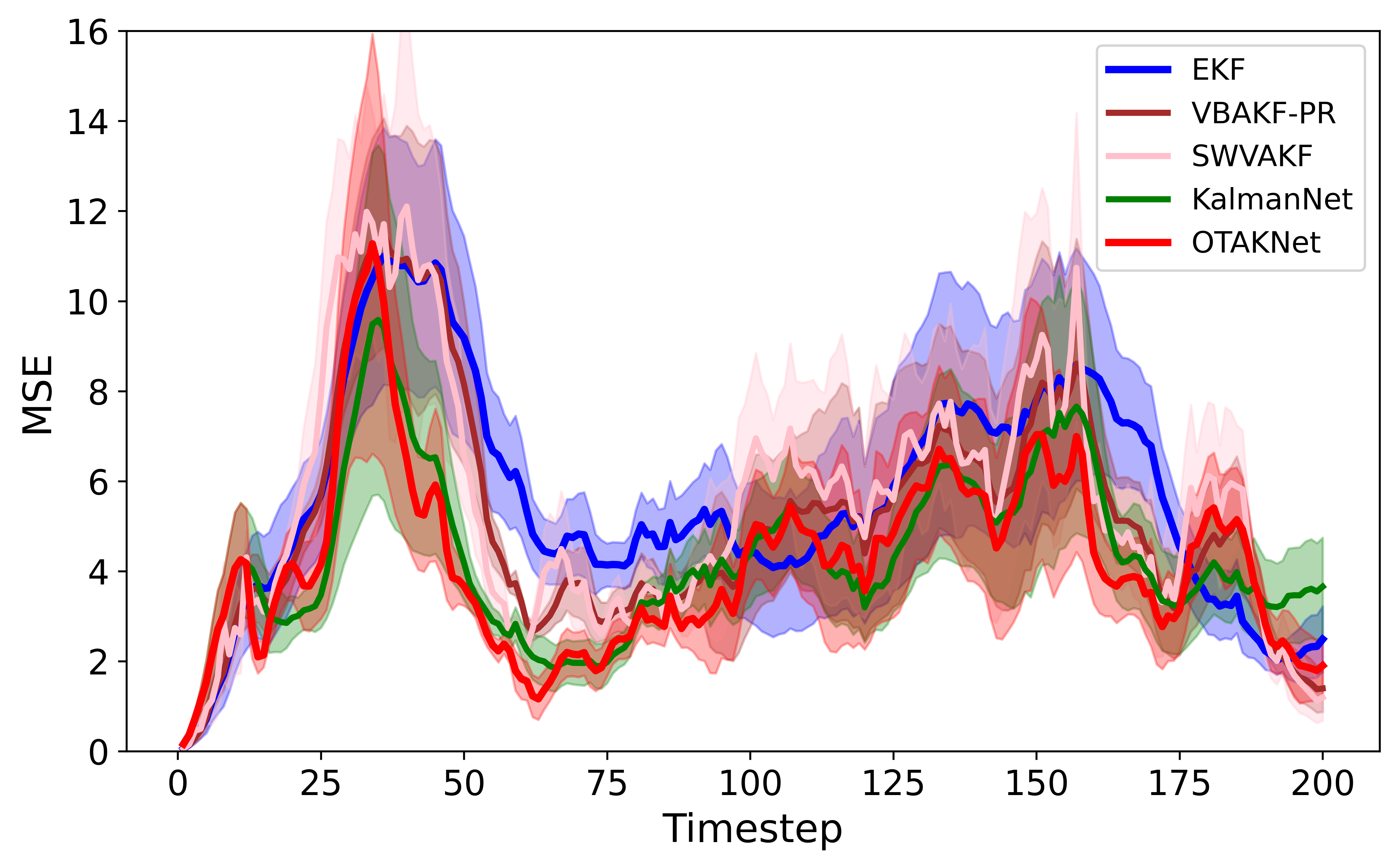}
    \caption{MSE curves under full training.}
    \label{fig:nclt_mse_full}
  \end{subfigure}
  \begin{subfigure}[b]{0.8\linewidth}
    \centering
    \includegraphics[width=\linewidth]{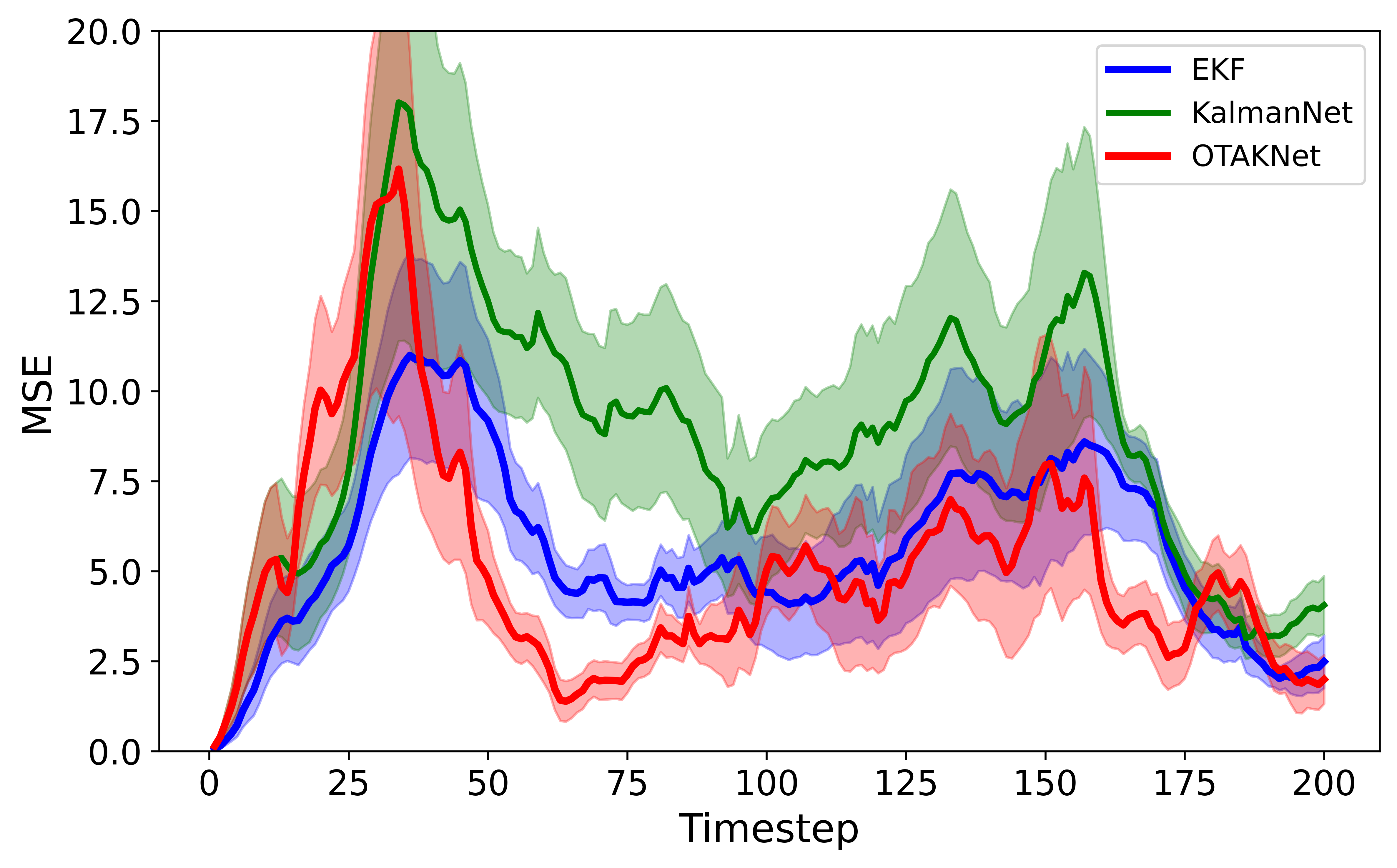}
    \caption{MSE curves under limited training.}
    \label{fig:nclt_mse_limited}
  \end{subfigure}

  \caption{MSE curves of online adaptive filters on the NCLT dataset under (a) full and (b) limited training.}
  \label{fig:nclt_two_panel}
\end{figure}
\begin{figure}
  \centering
  \includegraphics[width=0.8\linewidth]{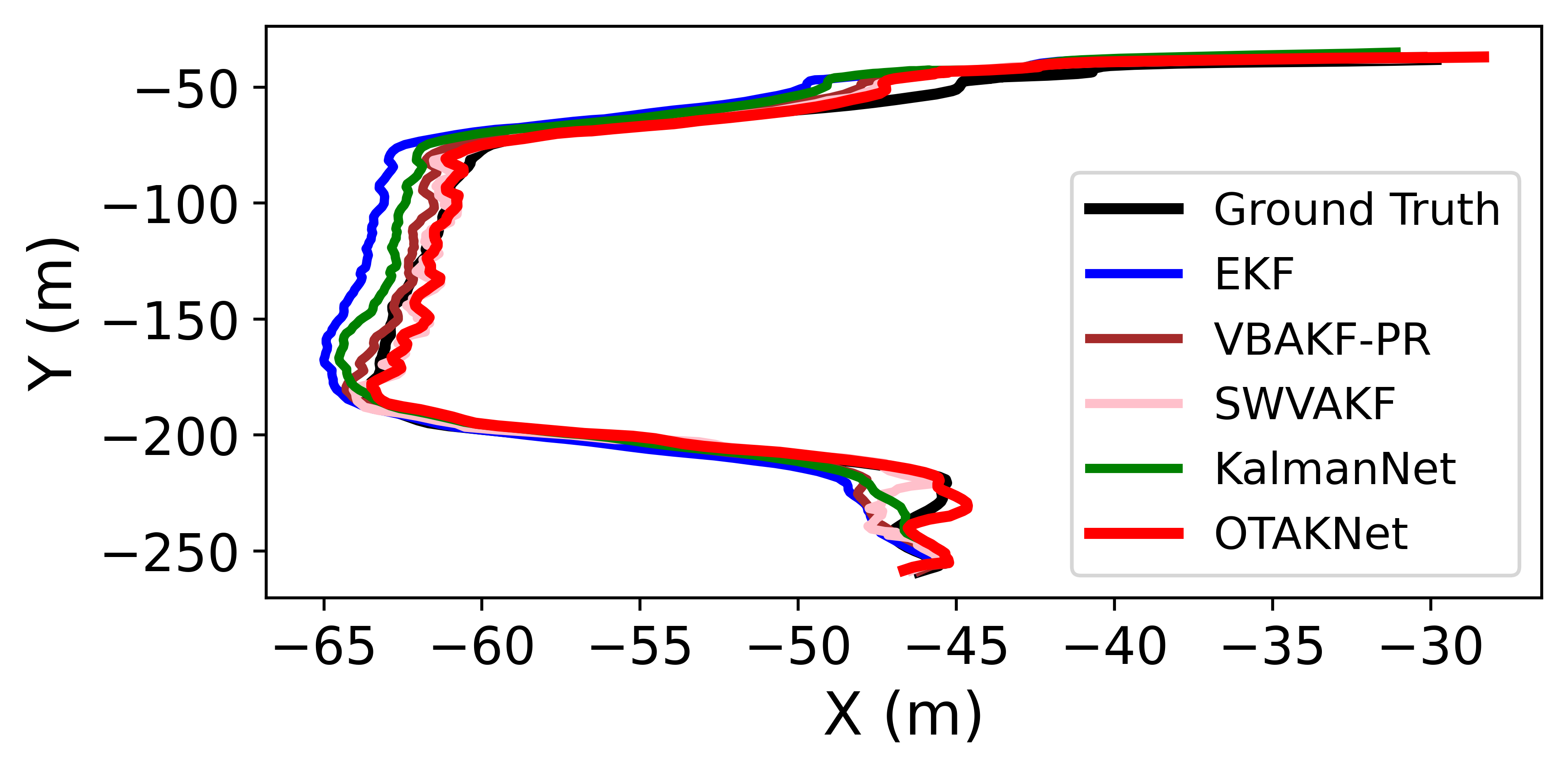}
  \caption{Estimated trajectory from NCLT dataset session.}
  \label{fig:nclt_NCLT_trajectory_comparison_xy}
\end{figure}

\begin{figure}
  \centering
  \includegraphics[width=0.8\linewidth]{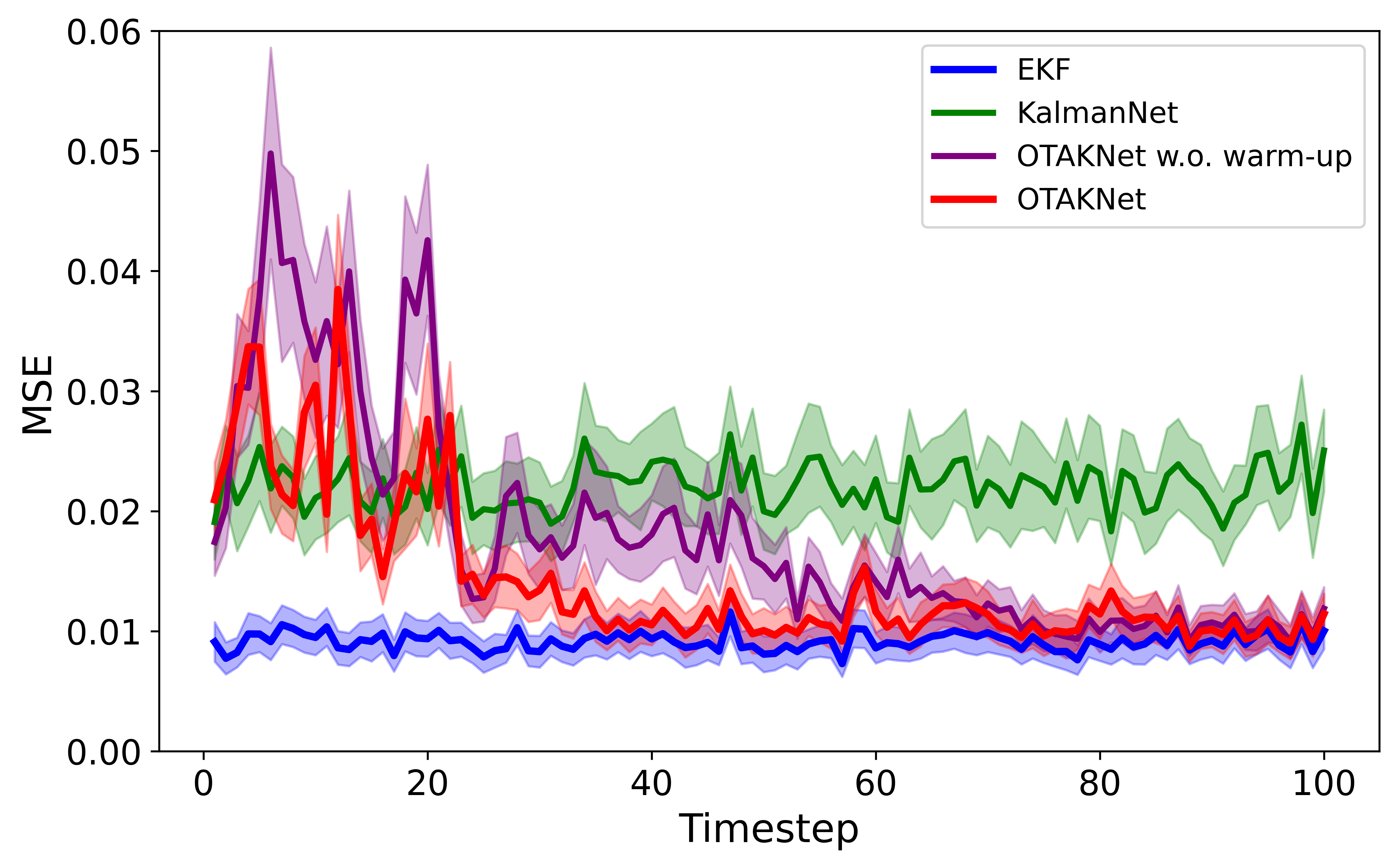}
  \caption{Linear warm‐up schedule for the learning rate. During the first \(W\) steps, the learning rate increases linearly to the base value \(\eta\), which helps mitigate estimation variability when temporal information about the drift is limited.}
  \label{fig:warmup_schedule}
\end{figure}

\subsection{Ablation Studies}

Table \ref{table_ablation_mse} reports ablation studies on synthetic and NCLT data to assess the contribution of each module.
In column 2 (w.o. warm-up in Eq.\ref{equ_warmup}) and 3 (w.o. OT in Eq.\ref{equ_otloss}, where the target distribution degenerates to a single point without windowed history), both modifications lead to increased MSE across all datasets, confirming the effectiveness of the OT-based loss and warm-up schedule in addressing noise-statistics drift.
Moreover, the analysis of MSE curves in Fig. \ref{fig:warmup_schedule} shows that OTAKNet exhibits instability in the early steps, possibly caused by limited residual history in the target distribution. The warm-up module mitigates this effect by stabilizing the adaptation process during initial filtering.

\begin{table}[H]
  \centering
  \setlength{\tabcolsep}{0.6 mm}
  \begin{tabular}{@{}l *{3}{c}@{}}
    \toprule
    \textbf{Dataset} & \textbf{OTAKNet} & \textbf{w.o. warm-up} & \textbf{w.o. OT} \\
    \midrule
    \cellcolor[HTML]{E2E6E1}Synthetic
      & \makecell[c]{\textbf{-18.45}\\\(\pm 0.41\)}
      & \makecell[c]{-17.48\\\(\pm 0.51\)}
      & \makecell[c]{-17.65\\\(\pm 0.62\)} \\
    \cellcolor[HTML]{E2E6E1}NCLT
      & \makecell[c]{\textbf{6.10}\\\(\pm 2.84\)}
      & \makecell[c]{6.28\\\(\pm 2.54\)}
      & \makecell[c]{6.29\\\(\pm 2.49\)} \\
    \bottomrule
  \end{tabular}
  \caption{Ablation Study: MSE [dB] for Warm‑up and OT Modules on Synthetic and NCLT Data}
  \label{table_ablation_mse}
\end{table}

\subsection{Computing Infrastructure and Computational complexity}\label{subsec: comput}

All experiments were conducted on a desktop equipped with an Intel Core i7-10700K CPU @ 3.80GHz, 16 GB RAM, and an NVIDIA GeForce RTX 2080 GPU. 

In the Lorenz Attractor scenario (the same experimental settings as described in manuscript), our method achieves inference times comparable to model-based online adaptive methods, as summarized in Table~\ref{table:inference_time}. While showing improved estimation performance.
Note that other learning-based methods, such as AKNet, MAML-KalmanNet, and KalmanNet are fine-tuned offline and do not support online adaptation. Thus their inference speed is reported here for reference only: 0.47 s, 0.37 s, and 0.47 s, respectively.

\begin{table}[ht]
  \centering
  \begin{tabular}{lcccc}
    \toprule
    Method & EKF & SWVAKF & VBAKF & OTAKNet \\
    \midrule
    Inference Time (s) & 0.53 & 4.28 & 19.81 & 21.02 \\
    \bottomrule
  \end{tabular}
  \caption{Online Adaptive Methods Inference Time (s) in Lorenz Attractor Scenarios}
  \label{table:inference_time}
\end{table}

\section{Conclusion}

A differentiable adaptive kalman filtering based on Optimal Transport, OTAKNet, is proposed to solve unlabeled, online test-time adaptation to noise-statistics drift.
It leverages optimal transport to connect the one-step predictive measurement likelihood with the adaptation of noise statistics.
This formulation avoids the need for offline fine-tuning commonly required in learning-based filters, enabling online parameter updates directly during inference.
The OT-based loss preserves geometric structure between distributions and provides meaningful, smooth gradients for stable optimization.
In the early stage, a warm-up module is introduced to alleviate the instability caused by limited temporal drift information.
Experimental results demonstrate that OTAKNet achieves robust adaptation under noise-statistics drift and effectively compensates for insufficient offline training.

\section{Appendix}

\subsection{Theoretical Analyze}\label{Theoretical Analyze}

\paragraph{Notation.}
Let $\hat x_{t\mid t-1}$ denote the filter’s one‐step state prediction at time $t$, and let
\[
e_k = y_k - \mathbf{h}(\hat x_{k\mid k-1}), 
\quad k = t-W+1,\dots,t,
\]
be the corresponding \emph{innovations} (residuals).  We write ``a.s.'' for \emph{almost‐surely} (i.e.\ with probability one).  The squared Wasserstein‐2 distance between two Gaussian measures $\mathcal{N}(m_1,S_1)$ and $\mathcal{N}(m_2,S_2)$ is given by
\[
W_2^2 = \|m_1 - m_2\|^2 + \mathrm{Tr}\!\bigl(S_1 + S_2 - 2\,(S_1^{1/2} S_2 S_1^{1/2})^{1/2}\bigr).
\]

\begin{lemma}[Empirical Innovation Covariance Consistency]\label{lem:innov_consistency}
Assume the innovation process $\{e_k\}$ is second‐order stationary and ergodic with
\[
\mathbb{E}[e_k] = \mu_e, 
\quad
\mathrm{Cov}(e_k) = \Sigma_e,
\]
independent of $k$.  Define the sample mean and covariance over the most recent $W$ innovations by
\[
\bar e_W = \frac{1}{W}\sum_{j=1}^W e_{t-W+j},
\qquad
\widehat\Sigma_W
= \frac{1}{W-1}\sum_{j=1}^W\bigl(e_{t-W+j}-\bar e_W\bigr)\bigl(e_{t-W+j}-\bar e_W\bigr)^{\!\top}.
\]
Then, as $W\to\infty$, the following hold almost‐surely (a.s.):
\[
\bar e_W \;\xrightarrow{\mathrm{a.s.}}\; \mu_e,
\qquad
\widehat\Sigma_W \;\xrightarrow{\mathrm{a.s.}}\;\Sigma_e.
\]
\end{lemma}

\begin{proof}
By the ergodic theorem for second‐order stationary processes \cite{supply_1931proof, supply_2017probability}, 
\[
\frac{1}{W}\sum_{j=1}^W e_{t-W+j}
\;\xrightarrow{\mathrm{a.s.}}\;
\mathbb{E}[e_k] = \mu_e,
\]
\[
\frac{1}{W}\sum_{j=1}^W e_{t-W+j}\,e_{t-W+j}^{\!\top}
\;\xrightarrow{\mathrm{a.s.}}\;
\mathbb{E}[e_k e_k^{\!\top}] = \Sigma_e + \mu_e\mu_e^{\!\top}.
\]
Using the identity
\[
\widehat\Sigma_W
= \frac{W}{W-1}\Bigl[\frac{1}{W}\sum e_j e_j^{\!\top} - \bar e_W\,\bar e_W^{\!\top}\Bigr],
\]
one immediately obtains $\widehat\Sigma_W \xrightarrow{\mathrm{a.s.}} \Sigma_e$.
\end{proof}

\begin{proposition}[Online innovation covariance adaptive consistency]\label{prop:cov_adapt}
Let a learning‐based filter be parameterized by $\theta$, so that its prior‐predictive mean and covariance at time $t$ are
\[
m_t(\theta) = \mathbf{h}(\hat x_{t\mid t-1}(\theta)),
\qquad
S_{t\mid t-1}(\theta)
= \mathbf{H}_t\,\Sigma_{t\mid t-1}(\theta)\,\mathbf{H}_t^{\!\top} + \mathbf{R}(\theta).
\]
At each time $t$, define the Gaussian measures
\[
\mu_{\mathrm src}(\theta)
= \mathcal{N}\bigl(m_t(\theta),\,S_{t\mid t-1}(\theta)\bigr),
\quad
\mu_{\mathrm tgt}
= \mathcal{N}\bigl(y_t,\,\widehat\Sigma_W\bigr),
\]
and the one‐step loss
\[
\ell_t(\theta)
= W_2^2\!\bigl(\mu_{\mathrm src}(\theta),\,\mu_{\mathrm tgt}\bigr).
\]
Assume:
\begin{enumerate}[label=\textbf{(A\arabic*)}]
  \item The innovations $\{e_k\}$ satisfy Lemma~\ref{lem:innov_consistency} (second‐order stationarity and ergodicity).
  \item $\ell_t(\theta)$ is continuously differentiable in $\theta$ and satisfies the Polyak–Łojasiewicz inequality.
\end{enumerate}
Then the gradient‐descent update
\[
\theta \;\leftarrow\; \theta - \eta\,\nabla_\theta \ell_t(\theta)
\]
converges linearly to a unique $\theta^*$ for which
\[
S_{t\mid t-1}(\theta^*) = \Sigma_e,
\]
i.e.\ the filter’s internal prior‐predictive covariance converges to the true innovation covariance.
\end{proposition}

\begin{proof}[Outline of Proof]
\begin{enumerate}[leftmargin=1.5em]
    \item By Lemma~\ref{lem:innov_consistency}, $\widehat\Sigma_W \xrightarrow{\mathrm{a.s.}} \Sigma_e$, hence the empirical target distribution satisfies
    \[
    \mu_{\mathrm{tgt}}^W \to \mathcal{N}(y_t,\Sigma_e).
    \]
    
    \item The squared Wasserstein-2 distance between two Gaussian distributions admits the following closed-form \cite{supply_class}:
    \[
    \ell_t(\theta)
    = \|m_t(\theta)-y_t\|^2
    + \mathrm{Tr}\Bigl(S_{t\mid t-1}(\theta) + \Sigma_e
    - 2\bigl(S_{t\mid t-1}(\theta)^{1/2} \Sigma_e\, S_{t\mid t-1}(\theta)^{1/2}\bigr)^{1/2} \Bigr).
    \]
    The second (covariance) term is always nonnegative, and vanishes if and only if $S_{t\mid t-1}(\theta) = \Sigma_e$.
    
    \item Under the Polyak–Łojasiewicz (PŁ) condition, gradient descent on $\ell_t(\theta)$ converges linearly to the unique global minimizer $\theta^*$, which in turn satisfies
    \[
    S_{t\mid t-1}(\theta^*) = \Sigma_e.
    \]
\end{enumerate}
\end{proof}

\bibliography{ref}

\end{document}